\documentclass[letterpaper, 10 pt, conference]{ieeeconf}

\IEEEoverridecommandlockouts                              %
\usepackage{amsmath,amsfonts,amssymb}

\usepackage{amsthm}
\usepackage{algorithm}
\usepackage{algpseudocode}
\usepackage{array}
\usepackage{graphicx}
\usepackage[noadjust]{cite}
\usepackage{bm}
\usepackage{tikz}
\usepackage{booktabs}
\usepackage{multirow}
\usepackage{verbatim}

\definecolor{Blue}{RGB}{23, 78, 166}
\definecolor{Red}{RGB}{165, 14, 14}
\definecolor{Orange}{RGB}{227, 116, 0}
\definecolor{Green}{RGB}{13, 101, 45}
\definecolor{MediumBlue}{RGB}{66, 103, 210}
\definecolor{MediumRed}{RGB}{234, 67, 53}
\definecolor{Yellow}{RGB}{251, 188, 4}
\definecolor{MediumGreen}{RGB}{52, 168, 83}
\definecolor{LightBlue}{RGB}{210, 227, 252}
\definecolor{LightRed}{RGB}{250, 210, 207}
\definecolor{LightYellow}{RGB}{254, 239, 195}
\definecolor{LightGreen}{RGB}{206, 234, 214}
\definecolor{LightGrey}{RGB}{241, 243, 244}
\definecolor{Grey}{RGB}{154, 160, 166}
\definecolor{Black}{RGB}{32, 33, 36}

\usepackage[breaklinks=true,colorlinks=true,linkcolor=Blue,citecolor=Blue,urlcolor=MediumBlue]{hyperref}
\usepackage[capitalize,noabbrev]{cleveref}

\makeatletter
\providecommand{\theHALG@line}{}
\renewcommand{\theHALG@line}{\thealgorithm.\arabic{ALG@line}}
\makeatother

\crefname{align}{Eq.}{Eqs.}
\crefname{theorem}{Theorem}{Theorems}

\newcommand{\methodname}{FS-MPC}

\newtheorem{theorem}{Theorem}
\newtheorem{lemma}{Lemma}

\newenvironment{delayedproof}[1]
 {\begin{proof}[\raisedtarget{#1}Proof of \Cref{#1}]}
 {\end{proof}}
\newcommand{\raisedtarget}[1]{%
  \raisebox{\fontcharht\font`P}[0pt][0pt]{\hypertarget{#1}{}}%
}
\newcommand{\proofref}[1]{\hyperlink{#1}{Proof of \Cref{#1}}}
\newcommand{\stkout}[1]{\ifmmode\text{\sout{\ensuremath{#1}}}\else\sout{#1}\fi}

\newif\ifincludenote
\includenotefalse
\ifincludenote
    \usepackage{todonotes}
    \newcommand{\cynote}[1]{\textcolor{Blue}{\textbf{Chaoyi Note:} #1}}
    \newcommand{\cytodo}[1]{\textcolor{Green}{\textbf{Chaoyi TODO:} #1}}
    \newcommand{\guanya}[1]{{\color{cyan}(Guanya: #1)}}
    \newcommand{\guannan}[1]{{\color{orange}(GQ: #1)}}
    \newcommand{\zeji}[1]{\textcolor{red}{Zeji: #1}}
    \newcommand{\john}[1]{\textcolor{red}{john: #1}}
\else
    \newcommand{\cynote}[1]{}
    \newcommand{\cytodo}[1]{}
    \newcommand{\guanya}[1]{}
    \newcommand{\guannan}[1]{}
    \newcommand{\zeji}[1]{}
    \newcommand{\john}[1]{}
\fi

\begin{document}

\title{\LARGE \bf
	Hybrid Feedback Sampling for Sample-Efficient \\ Model Predictive Control
}

\author{Chaoyi Pan$^{1,*}$, Zeji Yi$^{1,*}$,
	John Zhang$^{2}$, Zachary Manchester$^{2}$,
	Guannan Qu$^{1}$, Guanya Shi$^{3}$%
	\thanks{$^{*}$Equal contribution.}%
	\thanks{$^{1}$C. Pan, Z. Yi, and G. Qu are with the Department of
	Electrical and Computer Engineering, Carnegie Mellon University,
	Pittsburgh, PA 15213, USA
	({\tt\small \{chaoyip,zejiy,gqu\}@andrew.cmu.edu}).}%
	\thanks{$^{2}$J. Zhang and Z. Manchester are with the Department of
	Aeronautics and Astronautics, Massachusetts Institute of Technology,
	Cambridge, MA 02139, USA
	({\tt\small \{jzhang3,zacm\}@mit.edu}).}%
	\thanks{$^{3}$G. Shi is with the Robotics Institute, Carnegie Mellon
	University, Pittsburgh, PA 15213, USA
	({\tt\small guanyas@andrew.cmu.edu}).}%
}

\maketitle
\thispagestyle{empty}
\pagestyle{empty}

\begin{abstract}
	Thanks to its parallelizability and flexibility, sampling-based Model Predictive Control (MPC) has become widely popular for controlling real-world robotic systems.
	However, for high-dimensional and open-loop unstable dynamical systems, the required number of samples to improve the control sequence will grow exponentially with the horizon, leading to poor sample efficiency and numerical instability.
	This paper investigates the instability of shooting methods in sampling-based MPC and shows that the optimal sampling proposal distribution can be realized by sampling with an optimized feedback policy.
	We refer to this algorithm as \underline{F}eedback \underline{S}ampling MPC (FS-MPC).
	FS-MPC involves a hybrid sampling design which balances local and global search based on the system stability and the available computation budget.
	Our theoretical analysis shows that our hybrid sampling approach achieves faster convergence than standard MPPI and better optimality than standard feedback sampling.
	Empirically, in diverse contact-rich control tasks like humanoid loco-manipulation and dexterous manipulation, we show that FS-MPC successfully tackles dynamically unstable tasks where standard sample-based approaches struggle, and strictly outperforms feedback policies alone.
	Finally, we validate our method on humanoid robot locomotion and manipulation tasks in the real world.
\end{abstract}

\section{Introduction}

Sampling-based model-predictive control (MPC), such as model-predictive path integral control (MPPI), has emerged as a popular approach for solving challenging real-time optimal control problems~\cite{williamsAggressiveDrivingModel2016}
thanks to its straightforward implementation and natural parallelizability on modern GPUs.
However, most successful applications of sampling-based MPC are limited to open-loop stable tasks such as autonomous vehicles~\cite{yinImprovingModelPredictive2021}, quadrupeds~\cite{xueFullOrderSamplingBasedMPC2024, alvarez-padillaRealTimeWholeBodyControl2024}, or quasi-static manipulation~\cite{liDROPDexterousReorientation2024}. When applied to inherently unstable systems like quadrotors or humanoid robots, these sampling-based algorithms suffer from numerical instability during long open-loop rollouts (see \cref{fig:2d_drone_distribution}).

\begin{figure}[ht]
    \centering
    \includegraphics[width=\linewidth]{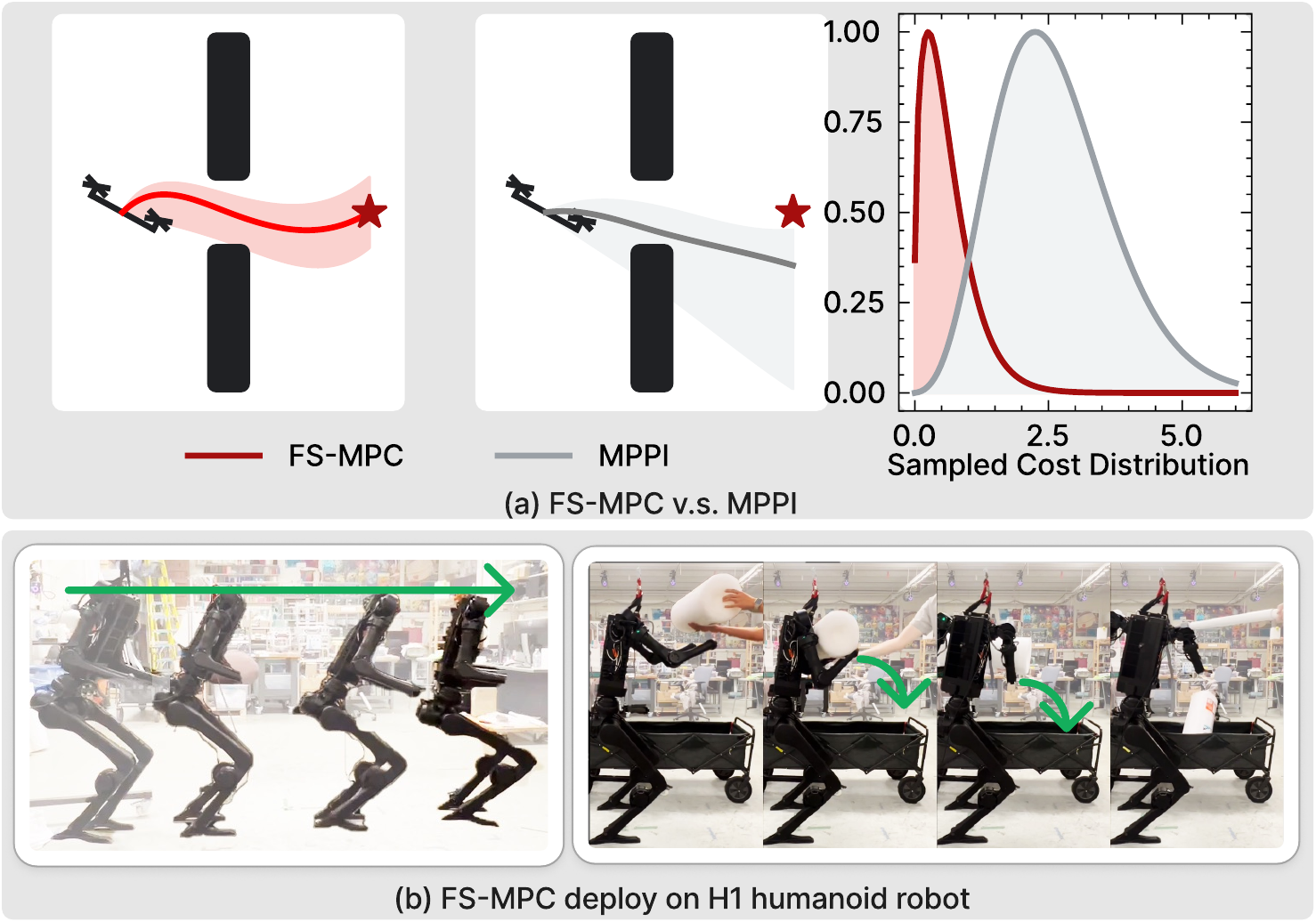}
    \vspace{-0.5cm}
    \caption{(a) Motivation of \methodname{}: showing the sampling distributions of \methodname{} (left) and standard MPPI (middle), and the distributions of their samples' costs (right), where the quadrotor is initialized at the same state. Given the long horizon navigation task with unstable quadrotor dynamics, feedback sampling can stabilize the system while MPPI's samples diverge. Therefore, the cost distribution of \methodname{} is significantly more optimal.
        (b) Deploy \methodname{} on the real humanoid robot. When planning over contact-rich and unstable humanoid dynamics, \methodname{} can control the robot with both stability and precision.
    }
    \label{fig:2d_drone_distribution}
\end{figure}

To overcome this challenge, efficient sampling strategies for unstable systems are needed. Existing work roughly falls into three camps. Improved covariance and filtering designs~\cite{balciConstrainedCovarianceSteering2022,williamsRobustSamplingBased2018,gandhiSafeImportanceSampling2023} help the sampler stay near stable trajectories, horizon reduction techniques such as value function shaping and control sequence interpolation~\cite{howellPredictiveSamplingRealtime2022} shrink the effective prediction horizon, and offline learning of sampling distributions~\cite{sacksLearningSamplingDistributions2023,yinSafeHorizonEfficient2025,adhauReinforcementLearningBased2024} produces better proposal distributions after random exploration.
All these approaches improve sampling efficiency but do not directly address the instability of the open-loop rollout step.
Because the fundamental instability originates there, these techniques still experience sampling divergence as the horizon lengthens, and the only way to maintain stability is to exponentially increase the sample count. We address this by introducing a sampling strategy that is designed with the rollout instability in mind.

In this paper, we introduce \underline{F}eedback \underline{S}ampling MPC, or \methodname{}.
We discover that the optimal sampling proposal distribution for linear time-variant (LTV) systems is equivalent to sampling from the optimal feedback policy for the same system.
We further extend this idea to general nonlinear systems that are not linearizable using a learned stabilizer.
Lastly, to ensure both sample efficiency and global convergence, \methodname{} employs a hybrid feedback sampling design that adjusts the feedback sample ratio based on system stability and available computation budget.

\noindent \textbf{Contributions:}
Our main contributions are three-fold:
\begin{enumerate}
    \item \textbf{Bounded Sample Complexity for Sampling-based MPC:} We analyze the sample inefficiency of MPPI (\Cref{thm:mppi_complexity}) and propose a hybrid feedback sampling rule for general nonlinear systems that improves sample efficiency compared to standard MPPI with feedback controller.

    \item \textbf{Practical Algorithms:} We provide a practical implementation of our feedback sampling approach, which is compatible with both iLQR and RL as feedback controllers for general nonlinear systems.

    \item \textbf{Experimental Validation:} We validate the method's performance in various high-dimensional, contact-rich, and unstable tasks (e.g., humanoid locomotion and manipulation) in both simulations (\Cref{subsec:exp_ilqr,subsec:exp_rl}) and the real world (\Cref{subsec:exp_realworld}), solving tasks that were previously intractable for sampling-based methods.
\end{enumerate}

\noindent \textbf{Organization:}
In \Cref{sec:formulation}, we begin by introducing the optimal control problem and sampling-based methods, which suffer from sample inefficiency.
In the algorithm section (\Cref{sec:stability}), we first look at the sample inefficiency of MPPI in \Cref{subsec:mppi_unstable} and explain why optimal proposal distribution design mitigates sample inefficiency and leads to a natural feedback sampling design in linearizable systems (\Cref{subsec:feedback_meta}).
We further extend this idea to general linearizable systems where a hybrid feedback sampling design is proposed to balance the sample efficiency and global convergence in \Cref{sec:hybrid_linearizable}.
Additionally, a learned feedback extension is proposed to handle systems where dynamics are not differentiable in \Cref{subsec:feedback_non_differentiable}.
Finally, in \Cref{sec:experiments}, we evaluate the method's performance in simulation (\Cref{subsec:exp_ilqr,subsec:exp_rl}) and on a real-world humanoid (Unitree H1) (\Cref{subsec:exp_realworld}).

\section{Related Work}

\subsection{Sampling-based MPC}
Sampling-based MPC~\cite{williamsAggressiveDrivingModel2016,williamsRobustSamplingBased2018} has been widely adopted in robotics control~\cite{pravitraL1AdaptiveMPPIArchitecture2020,sacksDeepModelPredictive2023,howellPredictiveSamplingRealtime2022} for its flexibility and scalability.
To improve the sampling efficiency, distribution improvement methods like CMA-ES~\cite{akimotoTheoreticalFoundationCMAES2012}, CEM~\cite{deboerTutorialCrossEntropyMethod2005} and learned distribution~\cite{sacksDeepModelPredictive2023} are proposed.
Sampling-based MPC has been widely used in non-convex control tasks like off-road driving~\cite{kimSmoothModelPredictive2022}, legged robot locomotion~\cite{xueFullOrderSamplingBasedMPC2024,alvarez-padillaRealTimeWholeBodyControl2024} and manipulation~\cite{howellPredictiveSamplingRealtime2022,liDROPDexterousReorientation2024}.
However, for unstable systems, the numerical instability of sampling-based MPC is a long-standing problem~\cite{kimSmoothModelPredictive2022,balciConstrainedCovarianceSteering2022} due to its stochastic nature and the instability of the open-loop trajectory sampling process.
\subsection{Gradient-based Nonlinear MPC}
Nonlinear MPC (NMPC) through formulations like
sequential quadratic programming (SQP,~\cite{siderisEfficientSequentialLinear2005}),
differentiable dynamic programming (DDP,~\cite{tassaControllimitedDifferentialDynamic2014}),
and iterative linear quadratic regulator (iLQR,~\cite{cariusTrajectoryOptimizationImplicit2018}) relies on differentiating through the optimal control problem.
NMPC has been successfully achieved real-time control for legged robot locomotion~\cite{parkHighspeedBoundingMIT2017, grandiaPerceptiveLocomotionNonlinear2022,neunertWholeBodyNonlinearModel2018}.
However, due to the high dimensionality of these systems and the inherent discontinuities of the contact dynamics, computing derivatives for these systems are non-trivial and not amenable to parallel hardware acceleration~\cite{carpentier2019pinocchio}.
To achieve real-time control, heuristics such as hierarchical design~\cite{martinezRealTimeNonlinearModel,scattoliniHierarchicalModelPredictive2007}, reduced-order models and fixed contact modes are often deployed~\cite{khazoomTailoringSolutionAccuracy2024,kuindersmaEfficientlySolvableQuadratic2014} to simplify or convexify the original nonconvex problem.

\subsection{Hybrid MPC}
To solve non-convex problems with stability, the combination of sampling and gradient-based methods has been explored for safety and robustness.
Our work is closely related to covariance steering \cite{yinImprovingModelPredictive2021,okamotoOptimalStochasticVehicle2018,ridderhofChanceConstrainedCovariance2020} used in safety-critical sampling-based MPC.
The idea of integrating gradient-based method to sampling is also explored in robust sampling \cite{yinTrajectoryDistributionControl2022,pravitraL1AdaptiveMPPIArchitecture2020} with tube constraint and hierarchical design.
This idea was originally proposed as Tube MPPI~\cite{williamsRobustSamplingBased2018}, which aims to stabilize the sampling distribution in a tube given model uncertainty.
Another design to improve the sampling efficiency is to combine offline learning with online sampling with learned sampling distribution \cite{sacksLearningSamplingDistributions2023}, dynamics model~\cite{adhauReinforcementLearningBased2024} and value function~\cite{sacksDeepModelPredictive2023} while still keeping the same open-loop rollout step.
\methodname{} differs from these methods in two ways:
Objective-wise, \methodname{} aims to achieve optimality/improvement over the standard sampling-based MPC, while these methods aim to improve robustness and safety, leading to conservative sampling.
Method-wise, \methodname{} is a meta-algorithm that can be instantiated with different feedback controllers, while the feedback in Tube-MPPI and Safe-MPPI requires differentiable dynamics that can be handled by gradient-based solvers.

\section{Formulation}
\label{sec:formulation}

\subsection{Notation}

For standard control notation, we use $x_t \in \mathbb{R}^n$ to denote the system state at time $t$ and $u_t \in \mathbb{R}^m$ to denote the control input at time $t$.
Control and state sequences are written in uppercase, e.g. $U = u_{1:H}$ and $X = x_{1:H+1}$, where $H$ is the prediction horizon.
For samples drawn from a distribution, the $i$-th sample is written as $[\cdot]_i$.
Specifically, the sampled control sequence is denoted $[V]_i = [v_{1:H}]_i$, where $i$ indicates the sample index.
The optimization-iteration index is $k$, and $U_k$ refers to the optimized control sequence produced at iteration $k$.

\subsection{Optimal Control Problem}

This paper considers the following discrete-time optimal control problem:

\begin{subequations}
    \begin{align}
        \label{eq:ocp}
        \min_{u_{1:H}} J(u_{1:H}) & = \sum_{h=1}^{H} c(x_{h}, u_{h}) + c_f(x_{H+1}),        \\
        \text{s.t.} \quad x_{h+1} & = f(x_{h}, u_{h}) \quad \forall h \in \{1, \ldots, H\}, \\
        x_{H+1}                   & \in \mathcal{X}, \ u_{1:H} \in \mathcal{V}
    \end{align}
\end{subequations}
where $x_h \in \mathbb{R}^n$ is the state with feasible set $\mathcal{X}$, $u_h \in \mathbb{R}^m$ is the control input with feasible set $\mathcal{V}$, $f : \mathbb{R}^n \times \mathbb{R}^m \to \mathbb{R}^n$ is the system dynamics, $c : \mathbb{R}^n \times \mathbb{R}^m \to \mathbb{R}$ is the running cost function, $c_f : \mathbb{R}^n \to \mathbb{R}$ is the terminal cost function, and $H$ is the prediction horizon. For simplicity, we denote control sequence with $U_k = u_{1:H}$ and state sequence with $X_k = x_{1:H+1}$. MPC algorithms will solve~\cref{eq:ocp} in a receding horizon manner.

\subsection{Sampling-based Optimal Control}
Sampling-based optimal control updates the control sequence $U = u_{1:H}$ in a zeroth-order manner as shown in~\cref{alg:sampling_meta}.
At each iteration, the algorithm samples $N$ control sequences $[V_k]_{1:N} \sim q(\cdot; U_k, \theta_k)$ from the proposal distribution.
The proposal distribution $q(\cdot; U_k, \theta_k)$ should be easy to sample from, where $\theta_k$ denotes its parameters.
For instance, in MPPI, the proposal distribution is a isotropic Gaussian: $q(\cdot; U_k, \theta_k) = \mathcal{N}(U_k, \sigma^2 I)$.
Then, the sampled control sequences are rolled out in the open-loop manner to get the system state sequence $[X_k]_i \gets \texttt{rollout}(x_0, [V_k]_i), i = 1, \ldots, N$ starting from the same initial state $x_0$.
Next, the cumulative cost of each sampled trajectory $([X_k]_i, [V_k]_i)$ is computed: $[J_k]_i \gets \texttt{cost}([X_k]_i, [V_k]_i), i = 1, \ldots, N$.
Finally, an update rule is applied to the control sequence, which depends on the specific algorithm.
For instance, in MPPI, the update rule is softmax which combines samples with lower cost with higher weights:
\begin{align}
    U_{k+1} = \frac{\sum_{i=1}^{N} \exp(-\frac{[J_k]_i}{\lambda})[V_k]_i}{\sum_{i=1}^{N} \exp(-\frac{[J_k]_i}{\lambda})}
    \label{eq:softmax_update}
\end{align}
where $\lambda$ is a temperature parameter. As $\lambda \rightarrow 0$, the update rule is equivalent to selecting the best control sequence leading to the smallest cost.
The subsequent discussion is agnostic to the specific update rule and is not limited to MPPI since we focus on analyzing each individual sample's quality.
The above iteration is repeated for $K$ times until convergence.
Running the algorithm in a receding-horizon manner yields a closed-loop sampling-based controller.

\begin{algorithm}[ht]
    \caption{General Sampling-Based Optimal Control}
    \label{alg:sampling_meta}
    \begin{algorithmic}[1]
        \State \textbf{Initialization:} $x_0 \sim \mathcal{X}_0, U_0 \leftarrow \mathbf{0}$

        \For{$k = 0$ to $K$} \Comment{Main optimization loop}
        \State $[V_k]_{1:N} \sim q(\cdot; U_k, \theta_k)$ \Comment{Sample control sequences}
        \State $[X_k]_{1:N} \gets \texttt{rollout}(x_0, [V_k]_{1:N})$ \Comment{Rollout system dynamics}
        \State $[J_k]_{1:N} \gets \texttt{cost}([X_k]_{1:N}, [V_k]_{1:N})$ \Comment{Compute cumulative cost}
        \State $U_{k+1} \gets \frac{\sum_{i=1}^{N} \exp(-\frac{[J_k]_i}{\lambda})[V_k]_i}{\sum_{i=1}^{N} \exp(-\frac{[J_k]_i}{\lambda})}$ \Comment{Update control sequence}
        \EndFor
        \State \Return $U_K$
    \end{algorithmic}
\end{algorithm}

\section{Feedback Sampling}
\label{sec:stability}

We formally introduce our feedback sampling framework by first examining the root cause of sampling inefficiency.
In \Cref{subsec:mppi_unstable}, we demonstrate that even in simplest unstable linear systems with optimal nominal control, open-loop shooting in sampling-based MPC still suffers from exponential sample complexity given long horizon and high-dimensionality.
Motivated by this, \Cref{subsec:feedback_meta} derives an optimal sampling distribution for linear time-variant systems (LTV), uncovering a duality between optimal sampling distribution and feedback policy.
Leveraging this insight, \Cref{sec:hybrid_linearizable} generalizes this approach to a practical algorithm with a hybrid feedback law for linearizable systems with non-convex objectives, balancing sample efficiency and global convergence.
A learned feedback extension is proposed to handle systems where dynamics are not differentiable in \Cref{subsec:feedback_non_differentiable}.
The final algorithm is summarized in \Cref{subsec:algo}.

\subsection{Sample Inefficiency of Sampling-based MPC}
\label{subsec:mppi_unstable}

\noindent \textbf{Open-loop shooting is unstable.}
MPPI's sample inefficiency stems from its open-loop rollout step in shooting procedure.
When the system is unstable, the sampled states can quickly diverge, leading to unstable behavior.

We illustrate this behavior with a simple LQR counterexample. Consider the linear dynamics $x_{h+1} = A x_h + B u_h$ and the quadratic cost $J = x_H^\top Q x_H + \sum_{h=0}^{H-1} (x_h^\top Q x_h + u_h^\top R u_h)$ with $Q, R \succeq 0$.
When the open-loop system is unstable, the probability of sampling a stable trajectory remains exponentially small, as formalized in \Cref{thm:mppi_complexity}.
\begin{theorem}[Exponential Sample Complexity of MPPI]
    \label{thm:mppi_complexity}
    Consider an controllable unstable linear system with $L_2$ norm $\|A\|_2 > 1$ and $\text{rank}(B) = \text{dim}(x)$. Let $[V_k]_i \sim \mathcal{N}(U_k, \sigma^2 I)$. Then the following results hold:

    \begin{enumerate}
        \item \emph{(Expected cost.)}
              \[
                  \mathbb{E}\left[ J([V_k]_i) \right] = \Theta(\|A\|_2^H)
              \]

        \item \emph{(Tail probability.)}
              For small $\varepsilon > 0$, we have
              \[
                  \Pr\bigl\{ J([V_k]_i) \le H \varepsilon \bigr\}
                  \;\le\;
                  O\!\bigl(\sqrt{H \varepsilon}\|A\|_2^{-2H}\bigr).
              \]

        \item \emph{(Sample complexity.)}
              To obtain at least one sample \( U \) satisfying \( J(U) \le  H \varepsilon \) with high probability \( 1 - \delta \), the required number of i.i.d.\ samples \( N \) satisfies
              \[
                  N = \Theta\!\left(
                  \frac{\|A\|_2^{2H-2}}{\sqrt{H \varepsilon}}
                  \log\frac{1}{\delta}
                  \right)
              \]
              for sufficient small $\varepsilon \sim o(1/H)$
    \end{enumerate}
\end{theorem}

See~\proofref{thm:mppi_complexity} for the detailed derivation.
\Cref{thm:mppi_complexity} implies that MPPI cannot reliably discover low-cost trajectories because almost all samples diverge and incur high cost.
Even when $U_k = U^*$, the best nominal control sequence, the exponential complexity remains.
To address this, we next optimize the sampling proposal distribution $q(\cdot)$ so that each sample has a higher chance of staying near the nominal trajectory. \Cref{subsec:feedback_meta} shows that this optimization naturally introduces feedback into the sampling process.

\subsection{Optimal Feedback Sampling for Linear Systems}
\label{subsec:feedback_meta}

\noindent \textbf{Optimal proposal distribution design for LTV systems.}
We motivate our feedback sampling design by studying the time-variant linear quadratic programming (TVLQR) problem.
We target to improve overall sample efficiency by optimizing proposal covariance matrix for a single sample expected cost:
\begin{align}
                       & \min_{\Sigma_k \succ 0} \mathbb{E}_{[V_k]_i \sim \mathcal{N}(U_k, \Sigma_k)}\left[ J([V_k]_i) \right]     \label{eq:optimal_covariance_matrix} \\
    \text{where} \quad & J([V_k]_i) = \sum_{h=1}^{H} \left([x_h]_i^\top Q_h [x_h]_i + [v_h]_i^\top R_h [v_h]_i \right) \notag                                   \\
                       & + [x_{H+1}]_i^\top Q_{H+1} [x_{H+1}]_i, \quad \text{det}(\Sigma_k) = 1 \notag
\end{align}
where $Q_h, R_h$ are the state and control cost matrices at step $h$, respectively. $[X]_i = [x_{1:H+1}]_i$ is the rollout trajectory of control sequence $[V_k]_i$, i.e. $[x_{h+1}]_i = A_h [x_h]_i + B_h [v_h]_i$.
\Cref{eq:optimal_covariance_matrix} is a matrix optimization problem that
admits a closed-form solution, as proved in CoVO-MPC~\cite{yiCoVOMPCTheoreticalAnalysis2024}.

\noindent \textbf{Optimized proposal distribution restricts divergence.}
Minimizing~\cref{eq:optimal_covariance_matrix} yields a sampler that generates trajectories with bounded cost variance, which stops trajectories' cost from exploding.

\begin{theorem}[Optimal Gaussian proposal removes horizon blow-up near an LQR optimum]
    \label{thm:stability_optimal_proposal}
    Consider a time-varying linear system with $H$ steps,
    \(
    x_{h+1} = A_h x_h + B_h u_h, h=0,\ldots,H-1,
    \)
    with quadratic cost as in~\eqref{eq:optimal_covariance_matrix}, and let
    $U^\star\in\mathbb{R}^{mH}$ be an optimal stacked open-loop control sequence.
    Consider Gaussian proposals with optimal mean, $V_k\sim\mathcal{N}(U^\star,\Sigma)$, and choose the covariance under a fixed-determinant constraint:
    \[
        \Sigma_k^\star
        \;=\;
        \arg\min_{\det(\Sigma)=1}\;
        \mathbb{E}\!\left[J(V_k)\right].
    \]
    Then the following results hold:
    \begin{enumerate}
        \item \emph{(Expected cost.)}
              \[
                  \mathbb{E}\left[ J([V_k]_i) \right]  = \Theta(H)
              \]
        \item \emph{(Variance.)}
              \[
                  \text{Var}\left[ J([V_k]_i)  \right] = \Theta(H)
              \]
        \item \emph{(Tail probability.)} There exists a constant $\varepsilon \sim \Theta(1)$ such that
              \[
                  \Pr\bigl\{ J([V_k]_i) \le H \varepsilon \bigr\} = \Theta(1)
              \]
        \item \emph{(Sample complexity.)}
              \[
                  N = \Theta(\log(1/\delta))
              \]
              to obtain at least one sample satisfying $J([V_k]_i)\le H\varepsilon$ with probability at least $1-\delta$
    \end{enumerate}
    None of these quantities exhibit exponential growth with horizon length $H$.
\end{theorem}

Proof is available in~\proofref{thm:stability_optimal_proposal}.
\Cref{thm:stability_optimal_proposal} indicates that the optimized covariance keeps the expected cost bounded and maintains a constant success probability regardless of horizon length.
Compared to the isotropic Gaussian from \Cref{thm:mppi_complexity}, the optimized sampler introduces off-diagonal correlations that make it more expressive across time and control dimensions, which largely improves the probability of sampling a low cost trajectory from exponentially small to constant, and leads to higher sample efficiency.

\noindent \textbf{Optimal proposal is equivalent to sampling with feedback.}
To understand why the optimal proposal covariance matrix removes the dependence on the horizon $H$, we now show the optimal proposal covariance matrix is equivalent to adding feedback when rolling out the control sequence in the sampling process.
Note that \cref{eq:optimal_covariance_matrix} is a quadratic program (QP) and can be solved efficiently with forward-backward dynamic programming, which shares the same structure as iLQR.

Consider the following iLQR controller with feedback and injected noise $w_h \sim \mathcal{N}(0, \sigma^2 I)$:
\begin{align}
    u_{h+1} & = - K_h x_h + Z_h w_h,
    \label{eq:optimal_feedback_design}
\end{align}
where design of $Z_h$ depends on the value function $P_h$ from iLQR, which can be computed by the following backward pass:
\begin{align*}
    K_{h} & = \bigl(R + B^\top P_{h+1} B\bigr)^{-1} \, B^\top P_{h+1} A \\
    Z_{h} & = \bigl(R + B^\top P_{h+1} B\bigr)^{-\tfrac{1}{2}}          \\
    P_{h} & = Q_{h} \;+\; K_{h}^\top
    \Bigl( P_{h+1} - A^\top P_{h+1} A \Bigr)
    K_{h}
\end{align*}

We next show that the distribution of $u_{h+1}$ induced by this noisy feedback matches $\mathcal{N}(U_k, \Sigma_k^*)$.

\begin{theorem}[Equivalence between feedback sampling and optimal sampling distribution]
    \label{thm:equivalence}
    Given a linear system with quadratic cost, consider the control distribution given by
    $$
        [u_{h+1}]_i = -K_h [x_h]_i + Z_h [w_h]_i, \quad [w_h]_i \sim \mathcal{N}(0,\sigma^2 I),
    $$
    which induces a Gaussian distribution on the full control sequence $[V_k]_i = [u_{1:H}]_i$.
    When $\lambda \to 0$ and $N \to \infty$, the covariance of the resulting Feedback MPPI distribution of $[V_k]_i$ achieves the optimum $\Sigma_k^*$, where:
    \begin{align}
        \Sigma_k^* = \arg \min_{\Sigma} \mathbb{E}_{[V_k]_i \sim \mathcal{N}(U_k, \Sigma)} \left[ J([X_k]_i, [V_k]_i) \right]. \label{eq:equivalence_optimal_covariance}
    \end{align}
\end{theorem}

\begin{figure}[htbp]
    \centering
    \includegraphics[width=\linewidth]{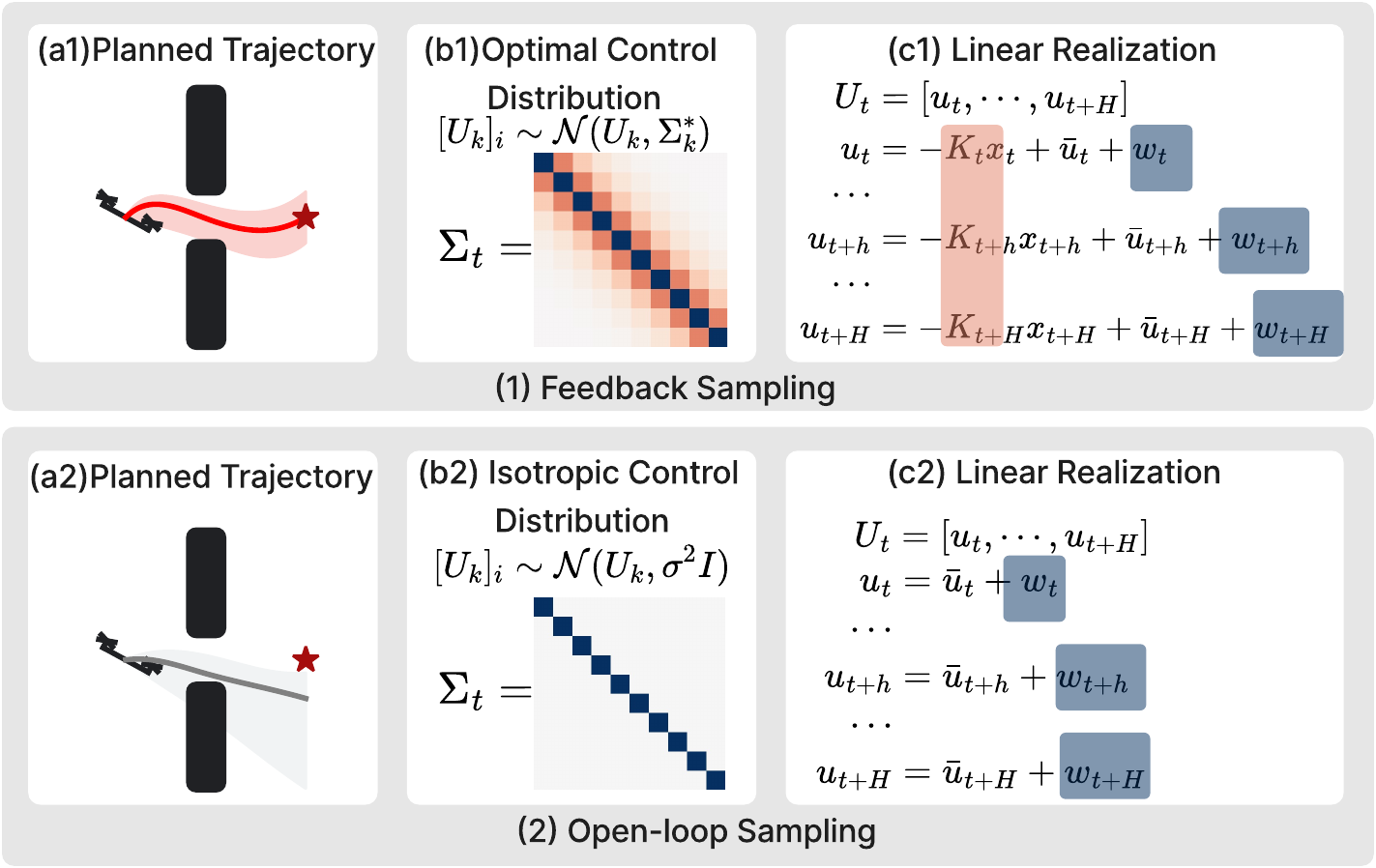}
    \caption{
        Visualization of shooting control sequence (a), its joint covariance matrix (b) and its linear realization (c). (1) sampling with iLQR feedback $K^*$ (c1) would lead to stable sampling distribution (b1) which has a rich pattern. (2) Standard MPPI sampling rule (c2) would lead to an isotropic Gaussian distribution (b2), which is not optimal. From the perspective of joint control distribution design, feedback sampling can be viewed as the state space equivalent realization of optimal covariance matrix $\Sigma_k^*$.}
    \label{fig:covo_explain}
\end{figure}

Proof is available in~\proofref{thm:equivalence}.
\Cref{thm:equivalence} states the equivalence between the following two ways to sample a control sequence $U_t$ used for rollout:
\begin{enumerate}
    \item \emph{Recursively} sample from a iLQR feedback controller with additional noise $w_h \sim \mathcal{N}(0, \sigma^2 I)$ or
    \item \emph{Directly} sample from a higher dimensional Gaussian distribution with covariance matrix $\Sigma_k^* \in \mathbb{R}^{mH \times mH}$.
\end{enumerate}
The idea is similar to dynamic programming: instead of directly sampling from a high dimensional Gaussian distribution, we can recursively sample from a lower dimensional Gaussian distribution and apply a feedback rule to achieve the same effect.
\Cref{fig:covo_explain} illustrates the equivalence between feedback sampling and optimal sampling distribution: when sampling from feedback sampling rule (c1), the resulting control sequence $[V_k]_i$ would have the same distribution as the one sampled from optimal sampling distribution (b1).
The connection between recursive sampling and direct sampling could also explains why \methodname{} is superior to standard MPPI: with feedback design, the resulting control sequence $U_t$ (b1) would have richer pattern than the one sampled from standard MPPI (b2).
Actually, what we have shown in~\Cref{thm:equivalence} is exactly the optimal sampling covariance design for linear stochastic systems in~\cite{yiCoVOMPCTheoreticalAnalysis2024} with minimum state space realization~\cite{hoEditorialEffectiveConstruction1966,picciStochasticRealizationGaussian1976,akaikeMarkovianRepresentationStochastic1975}.

\subsection{Hybrid Feedback Sampling for Linearizable Systems}
\label{sec:hybrid_linearizable}

The previous subsection established optimal covariance design for LTV systems with quadratic costs and demonstrated its equivalence to optimal iLQR feedback design. 
In this subsection, we focus on nonlinear systems. 
While the provably optimal sampling strategy for nonlinear system is generally difficult to obtain, the feedback design of previous subsection can be extended to nonlinear systems which we present below and validate in \cref{sec:experiments}, assuming we can still differentiate the dynamics and cost. 

\noindent\textbf{Extension to nonlinear systems via linearization.}
A naive approach involves approximating the nonlinear dynamics with a linear model (e.g., via finite differencing or use differentiable simulators) and applying the derived sampling design directly.
However, such linearization is valid only locally.
Consequently, a covariance design based solely on local linearization may converge to a local minimizer, failing to explore the broader state space in non-convex landscapes of the nonlinear optimal problem.

\begin{figure}[htbp]
    \centering
    \includegraphics[width=\linewidth]{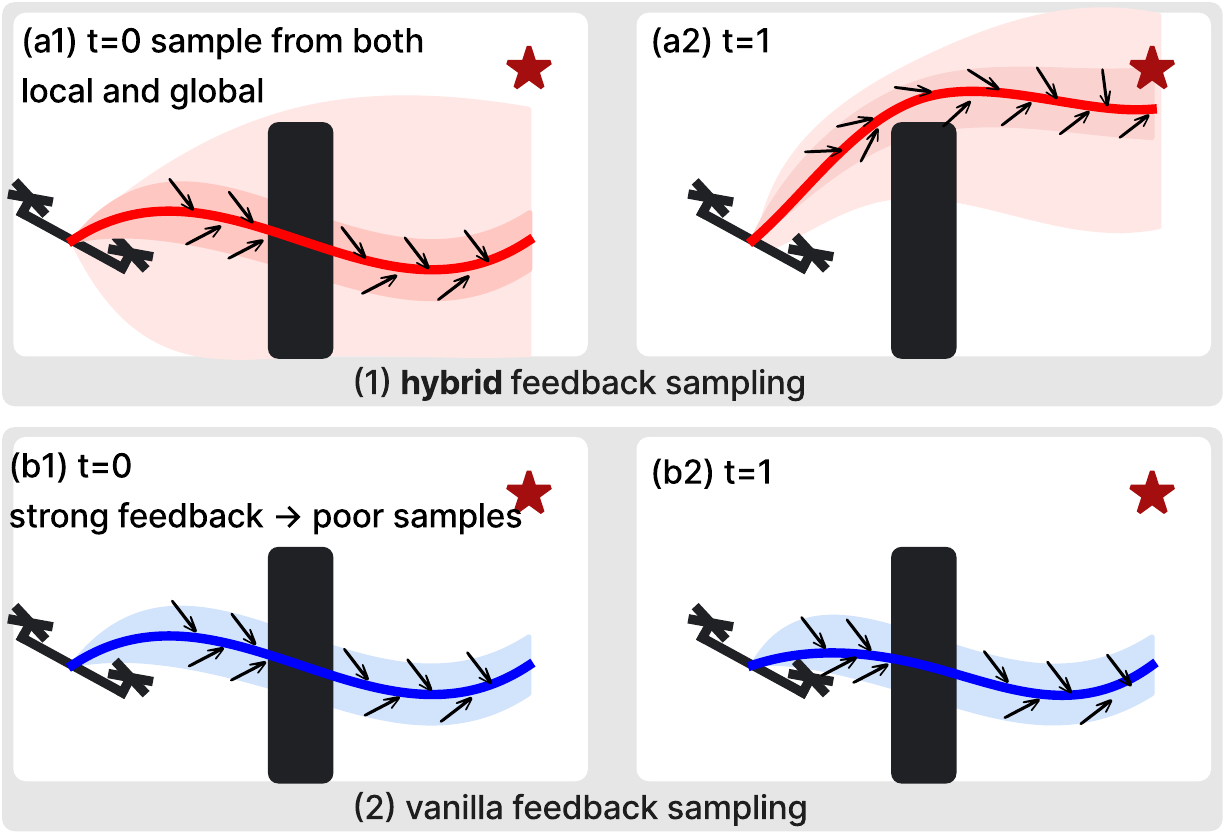}
    \caption{Hybrid feedback sampling balances local and global search. (1) Hybrid feedback sampling draws from both local and global proposal distributions, enabling escape from local minima. (2) Standard feedback sampling relies solely on the local distribution, often failing to traverse high-cost barriers.}
    \label{fig:adaptive_feedback_explain}
\end{figure}

To address this, we propose a \textit{hybrid feedback sampling} strategy that combines gradient-based and sampling-based methods, balancing local exploitation (guided by the linearized feedback) and global exploration.
As illustrated in~\Cref{fig:adaptive_feedback_explain}, we define the hybrid proposal distribution as a mixture model:
\begin{align}
    q_\alpha(U_k) = \alpha \underbrace{q(\cdot; U_k, \Sigma_k^*)}_{\text{Local Search}} + (1-\alpha) \underbrace{q(\cdot; U_k, \sigma^2 I)}_{\text{Global Search}},
    \label{eq:hybrid_proposal}
\end{align}
where $\Sigma_k^*$ is the covariance matrix optimized via~\cref{eq:optimal_covariance_matrix}, corresponding to locally linearized dynamics and local quadratic approximation of the cost.
The second term in~\cref{eq:hybrid_proposal}, possessing larger variance according to~\cref{thm:mppi_complexity}, can cover the global landscape better at the cost of worse sample complexity.
The scalar $\alpha \in [0, 1]$ determines the mixing ratio between the informative local feedback and the isotropic global noise.
A critical challenge lies in selecting the optimal $\alpha$ to balance these competing objectives.

\noindent \textbf{Mixed proposal distribution design.}
We design the mixing ratio $\alpha$ to ensure that the risk of missing a local minimum is bounded, guaranteeing sensible progress at each iteration while preventing divergence.

We observe that the sample cost variance under the optimized proposal $q(\cdot; U_k, \Sigma_k^*)$ is bounded, specifically $\text{Var}[J(V_k)] = \Theta(mH)$ according to~\cref{thm:stability_optimal_proposal}, where $m$ is the control dimension and $H$ is the horizon.
Given single sample's probability of reaching local minimum is a constant $P(J([V_k]_i) - J(U_\text{local}) < \epsilon) = \frac{C_\text{local}}{mH}$, the chance of at least one sample is near local optimal is $P(\min_i J([V_k]_i) - J(U_\text{local}) < \epsilon) = 1 - (1-\frac{C_\text{local}}{mH})^{N_\text{local}}$. Let $P(\min_i J([V_k]_i) - J(U_\text{local}) < \epsilon) > 1 - \delta$, we have

\begin{align}
    N_\text{local} \geq \frac{mH}{C_\text{local}} \log{\frac{1}{\delta}} := N_\text{local}^*
\end{align}

Once this local condition is satisfied, additional samples can be allocated to global exploration without degrading the chance of sampling local minima.
In practice, we determine $\alpha$ dynamically based on the total sample count $N$:
\begin{align}
    \alpha = \min \left\{\frac{N_\text{local}^*}{N}, 1 \right\},
    \label{eq:optimal_alpha}
\end{align}
Under this scheme, the algorithm prioritizes the optimized feedback sampler ($\alpha \to 1$) to ensure the local minimizer is discovered in the low-sample regime. As $N$ increases, the surplus samples are directed toward standard sampling without feedback to enhance global exploration.

\subsection{Hybrid Feedback Sampling for Non-differentiable Systems}
\label{subsec:feedback_non_differentiable}
In contact-rich control problems, system dynamics may become non-differentiable, rendering gradient-based methods (e.g., DDP or iLQR) inapplicable.
A key advantage of sampling-based MPC is its derivative-free nature. Therefore, we consider cases where the feedback proposal distribution cannot be optimized via gradient-based methods.

\noindent \textbf{Extension to non-differentiable systems.}
Even in the absence of gradients, it is possible to obtain a local feedback policy using RL methods, such as PPO~\cite{schulmanProximalPolicyOptimization2017}. The objective of such a stochastic policy matches that of our optimal covariance formulation~\cref{eq:equivalence_optimal_covariance} with a generalized proposal distribution parameterized by $\theta$:
\begin{align}
    \min_{\theta} \mathbb{E}_{[V_k]_i \sim q(\cdot; U_k, \theta)} \left[ J([X_k]_i, [V_k]_i) \right].
    \label{eq:optimal_proposal_distribution_general}
\end{align}
We extend our framework to support pre-trained nonlinear neural policies. Let $\pi_{\theta^*}$ denote a trained RL \emph{stochastic policy} parameterized by $\theta^*$. This implicitly defines a local proposal distribution $q(\cdot; U_k, \theta^*)$ via the following recursive sampling procedure:
\begin{align}
    [v_h]_k \sim \pi_{\theta^*}([x_h]_k), \quad [x_{h+1}]_k = f([x_h]_k, [v_h]_k).
\end{align}
The generalized hybrid proposal distribution becomes:
\begin{align}
    q_\alpha(U_k) = \alpha \underbrace{q(\cdot; U_k, \theta^*)}_{\text{Local Search}} + (1-\alpha) \underbrace{q(\cdot; U_k, \sigma^2 I)}_{\text{Global Search}},
\end{align}
where $\theta^*$ represents either the parameters of the RL policy or the covariance matrix from iLQR, depending on the availability of gradients.
For optimal $\alpha$ choice, we follow the same principle as the one in~\Cref{sec:hybrid_linearizable}.
We assume that the feedback policy is able to generate trajectory's cost with bounded variance.
Therefore, we use a fixed number of samples to find the local minimum, where $\alpha$ is also given by~\cref{eq:optimal_alpha}.

\subsection{Algorithm Realization}
\label{subsec:algo}
The full Hybrid Feedback Sampling algorithm is summarized in~\Cref{alg:feedback_sampling}. Compared to standard MPPI, our approach introduces two major improvements:
\begin{enumerate}
    \item A locally optimized sampling parameter that enhances sample efficiency.
    \item A hybrid proposal distribution that balances local exploitation and global exploration.
\end{enumerate}
When applied in a receding horizon setting, we refer to this method as Feedback Sampling MPC (FS-MPC).

\begin{algorithm}[ht]
    \caption{Feedback Sampling for Optimal Control}
    \label{alg:feedback_sampling}
    \begin{algorithmic}[1]
        \State \textbf{Input:} Initial state $x_0$, initial control sequence $U_0$, parameters $\lambda, \alpha_0$.
        \State Determine mixing ratio $\alpha$ and split $N$ into $N_{\text{local}}, N_{\text{global}}$
        \For{$k = 0$ to $K$} \Comment{Optimization iterations}
        \State $\theta_k^* / \Sigma_k^* \gets \text{solve}~\eqref{eq:optimal_proposal_distribution_general}$ or $\eqref{eq:optimal_covariance_matrix}$ \Comment{\textcolor{Blue}{Optimize sample parameter}} \label{alg:step_cov}
        \State \textbf{Sample} $[V_k]_{1:N_{\text{local}}} \sim q(\cdot; U_k, \theta_k^* / \Sigma_k^*)$ \Comment{\textcolor{Blue}{Local proposal}}
        \State \textbf{Sample} $[V_k]_{N_{\text{local}}+1:N} \sim q(\cdot; U_k, \sigma_k^2 I)$ \Comment{\textcolor{Blue}{Global proposal}}
        \State $[X_k]_{1:N} \gets \texttt{rollout}(x_0, [V_k]_{1:N})$ \Comment{Forward dynamics}
        \State $[J_k]_{1:N} \gets \texttt{cost}([X_k]_{1:N}, [V_k]_{1:N})$ \Comment{Evaluate trajectories}
        \State $U_{k+1} \gets \frac{\sum_{i=1}^{N} \exp(-[J_k]_i/\lambda)[V_k]_i}{\sum_{i=1}^{N} \exp(-[J_k]_i/\lambda)}$ \Comment{Update mean control}
        \EndFor
        \State \Return $U_K$
    \end{algorithmic}
\end{algorithm}

\section{Experiments}
\label{sec:experiments}

In this section, we demonstrate \methodname{}'s sample efficiency with optimized sampling distribution as well as its performance in real-world deployment.
Our experiments show that \methodname{} can significantly improve over standard MPPI by $43.4\%$ on cumulative cost with a suboptimal feedback controller.

\subsection{\methodname{} for linearizable systems}
\label{subsec:exp_ilqr}

We start by evaluating \methodname{} on linearizable systems where the optimal sampling distribution can be approximated by a Gaussian distribution.
To assess whether \methodname{} can generate better samples with an optimized feedback sampling strategy,
we evaluate it on a set of unstable control tasks ranging from lower-dimensional classical control tasks to high-dimensional legged robot manipulation tasks, as shown in~\cref{fig:tasks}.
For unstable tasks like Quadrotor and Quadruped Hill, the robot might hit an obstacle and fall over, which requires not only stabilization but also recovering when it gets stuck. For manipulation tasks like H1-2 PNP and Allegro, the controller needs to stabilize both the robot and the object while reasoning over contact, making these tasks challenging for both gradient-based and sampling-based MPC.

As shown in~\cref{tab:cost_mjpc}, \methodname{} realizes the benefits of both gradient-based and sampling-based MPC:
it steadily outperforms standard sampling-based MPC across all tasks thanks to the feedback design, which limits the divergence of samples illustrated in~\cref{fig:tasks}.
Meanwhile, \methodname{} maintains comparable performance on contact-free tasks like Acrobot and Quadrotor while substantially outperforming iLQR on long-horizon contact-rich tasks such as H1-2 PNP and Allegro, where iLQR fails because of noisy gradients and local minima.

\begin{figure}[htbp]
    \centering
    \includegraphics[width=\linewidth]{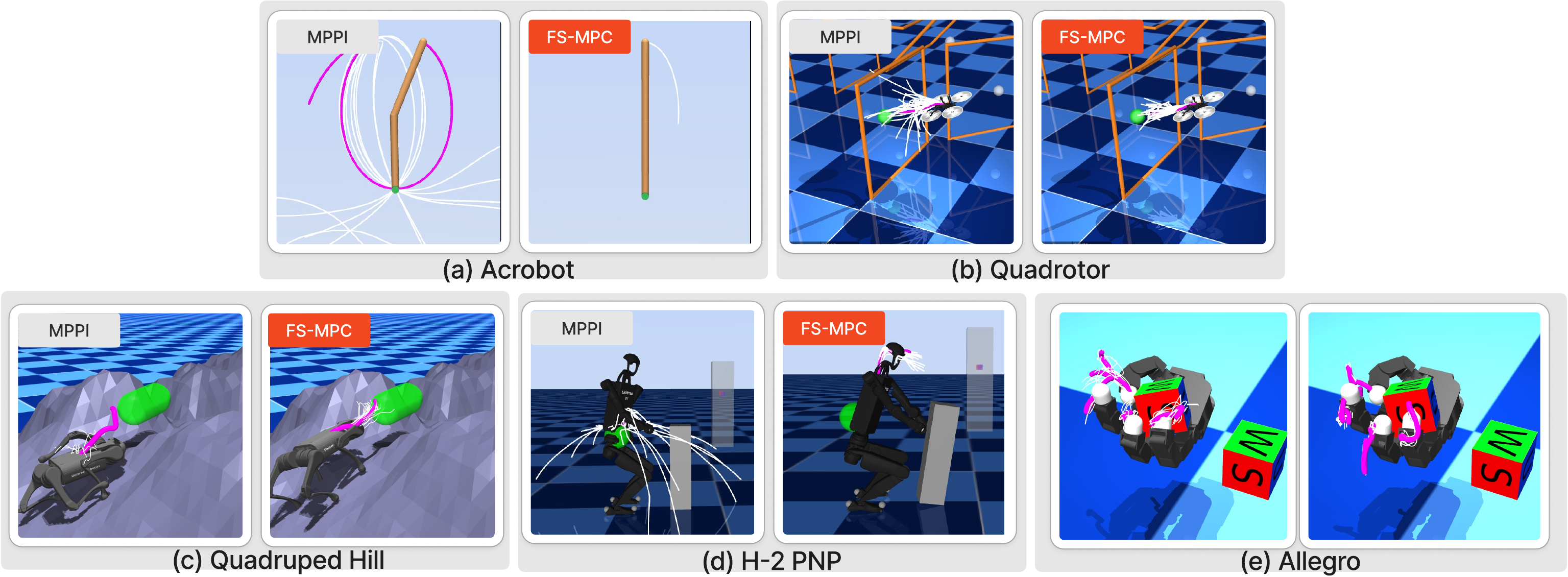}
    \caption{Task overview over different controllers.
        (a) Acrobot: the controller needs to swing up the pendulum from the bottom to the top and then balance it at an unstable equilibrium.
        (b) Quadrotor: the controller needs to track a desired trajectory and fly through a few gates. The quadrotor is unstable and colliding with a gate can lead to a crash.
        (c) Quadruped Hill: the controller needs to track desired key points to navigate a hill. Due to the uneven terrain, the quadruped would fall and need to get up.
        (d) H1-2 PNP: the controller needs to do whole-body control for a 23-DoF humanoid robot to pick up and place a box on the target location. The system involves both an unstable system (the humanoid robot) and long-horizon reasoning over the contact.
        (e) Allegro: a in-hand reorientation task with allegro hand. The robot needs to reorient the object to a random target orientation.
        In all tasks, the white line is the sampled trajectory, and the pink line is the planned trajectory. }
    \label{fig:tasks}
\end{figure}

\begin{table}[!htb]
    \centering
    \vspace{2mm}
    \resizebox{\columnwidth}{!}{
        \begin{tabular}{lcccc}
            \toprule
            Controllers             & \texttt{iLQR}                   & \texttt{MPPI}                   & \texttt{F-MPPI}                 & \texttt{FS-MPC}                          \\
            \midrule
            \texttt{Acrobot}        & {\scriptsize $104.55 \pm 0.63$} & {\scriptsize $117.16 \pm 0.20$} & {\scriptsize $122.12 \pm 7.45$} & {\scriptsize $\mathbf{76.38 \pm 15.24}$} \\
            \texttt{Quadrotor}      & {\scriptsize $8.44 \pm 0.37$}   & {\scriptsize $13.36 \pm 9.77$}  & {\scriptsize $8.78 \pm 0.20$}   & {\scriptsize $\mathbf{7.99 \pm 0.40}$}   \\
            \texttt{Quadruped Hill} & {\scriptsize $3.39 \pm 0.70$}   & {\scriptsize $2.79 \pm 0.02$}   & {\scriptsize $2.68 \pm 0.17$}   & {\scriptsize $\mathbf{2.51 \pm 0.04}$}   \\
            \texttt{Allegro}        & {\scriptsize $16.86 \pm 2.40$}  & {\scriptsize $12.22 \pm 4.84$}  & {\scriptsize $9.27 \pm 1.28$}   & {\scriptsize $\mathbf{6.50 \pm 0.38}$}   \\
            \texttt{H1-2 PNP}       & {\scriptsize $1.72 \pm 0.06$}   & {\scriptsize $1.55 \pm 0.15$}   & {\scriptsize $2.00 \pm 0.43$}   & {\scriptsize $\mathbf{1.06 \pm 0.10}$}   \\
            \bottomrule
        \end{tabular}
    }
    \caption{Cost comparison between controllers with iLQR feedback. From top to bottom: tasks are sorted from lower-dimensional to higher-dimensional, and from contact-free to contact-rich. \methodname{} can consistently outperform MPPI and iLQR.}
    \label{tab:cost_mjpc}
\end{table}

\subsection{\methodname{} for non-differentiable systems}
\label{subsec:exp_rl}
When the system is non-differentiable, a stable proposal distribution requires more complex parameterization.
\methodname{} supports using RL policies as feedback, which can be trained without gradient information.
Here, we evaluate \methodname{} with a PPO-trained neural-network policy as feedback using the Brax~\cite{freemanBraxDifferentiablePhysics2021} simulator.
The PPO controller is trained with $3$M steps until convergence and the sampling number across all tasks is set to $N=32$.

\Cref{tab:rl_cost_comparison} shows the cost comparison between different controllers.
For nonlinear feedback controllers, the role of hybrid feedback is more significant.
Interestingly, when comparing the performance of F-MPPI with the nominal PPO controller, F-MPPI performs even worse than the nominal controller.
This is due to the instability introduced by injecting noise into the nominal controller and the inability to explore the landscape because of the strong feedback, leading to a noisy PPO policy.
On the contrary, \methodname{} can still outperform both the nominal PPO controller and MPPI thanks to the hybrid feedback design:
it balances local exploitation and global exploration to effectively search the landscape.

\begin{table}[ht]
    \centering
    \begin{tabular}{lcc}
        \toprule
        Tasks                           & \texttt{humanoidrun}       & \texttt{humanoidstandup}   \\
        \midrule
        \texttt{iLQR}                   & -                          & -                          \\
        \texttt{PPO}                    & $-12.34 \pm 0.01$          & $-24.76 \pm 0.01$          \\
        \texttt{MPPI}                   & $-12.50 \pm 0.00$          & $-27.75 \pm 0.47$          \\
        \texttt{F-MPPI (w/ PPO)}        & $-10.07 \pm 0.22$          & $-21.14 \pm 0.06$          \\
        \texttt{\methodname{} (w/ PPO)} & $\mathbf{-15.65 \pm 0.10}$ & $\mathbf{-29.41 \pm 0.55}$ \\
        \bottomrule
    \end{tabular}
    \caption{Rollout trajectory cumulative cost on the humanoid running/standing up task.}
    \label{tab:rl_cost_comparison}
\end{table}

\subsection{Deployment on real robot}
\label{subsec:exp_realworld}
Another advantage of sampling with feedback is that it preserves the robustness of the sampling procedure.
We deploy \methodname{} on a full-size humanoid robot (Unitree H1) for whole-body control to conduct locomotion and manipulation tasks as shown in~\cref{fig:realworld}.
For locomotion, we task the robot with walking from one end of the room to the other while tracking a linear trajectory.
\Cref{fig:realworld_tracking_error} and \Cref{tab:realworld_tracking_error} show the tracking error of MPPI, iLQR, and \methodname{}.
For unstable systems like the humanoid robot, MPPI failed to stabilize the robot due to the high variance of the sampling procedure while \methodname{} finishes the task thanks to the feedback design.
Compared with iLQR, \methodname{} shows promising performance but only marginal improvement.
We believe this stems from the limitations of our current state estimation system, which uses a MoCap system to get the robot state.
The MoCap system's limited accuracy and communication delay failed to provide precise contact information, leading to high variance in cost. Further discussion can be found in~\Cref{sec:limitations}.

\begin{figure}[!htb]
    \centering
    \vspace{2mm}
    \includegraphics[width=\linewidth]{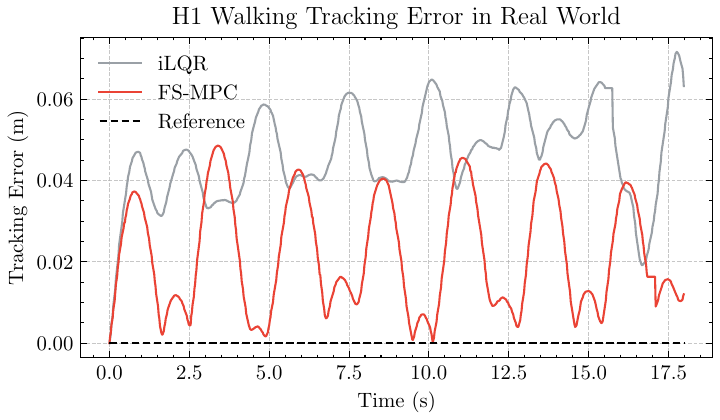}
    \caption{Tracking error of iLQR and FS-MPC on real robot. }
    \label{fig:realworld_tracking_error}
\end{figure}

\begin{table}[ht]
    \centering
    \begin{tabular}{lccc}
        \toprule
        \texttt{Methods}   & \texttt{MPPI} & \texttt{iLQR}   & \texttt{FS-MPC} \\
        \midrule
        Tracking Error (m) & \text{Failed} & $0.93 \pm 1.00$ & $0.70 \pm 0.66$ \\
        \bottomrule
    \end{tabular}
    \caption{Tracking error comparison over different controllers in the real world. MPPI failed due to the high variance of the sampling procedure. \methodname{} can still stabilize the robot thanks to the feedback design.}
    \label{tab:realworld_tracking_error}
\end{table}

For manipulation, we task the robot with holding a cylinder and placing it into a cart.
Both tasks require the controller to reason over the contact while maintaining balance.
The MPPI policy generates noisy and unstable behavior, leading the robot into an unrecoverable crash.
Meanwhile, \methodname{} can still generate stable and robust behavior thanks to the feedback design.

\begin{figure}[!htb]
    \centering
    \vspace{2mm}
    \includegraphics[width=\linewidth]{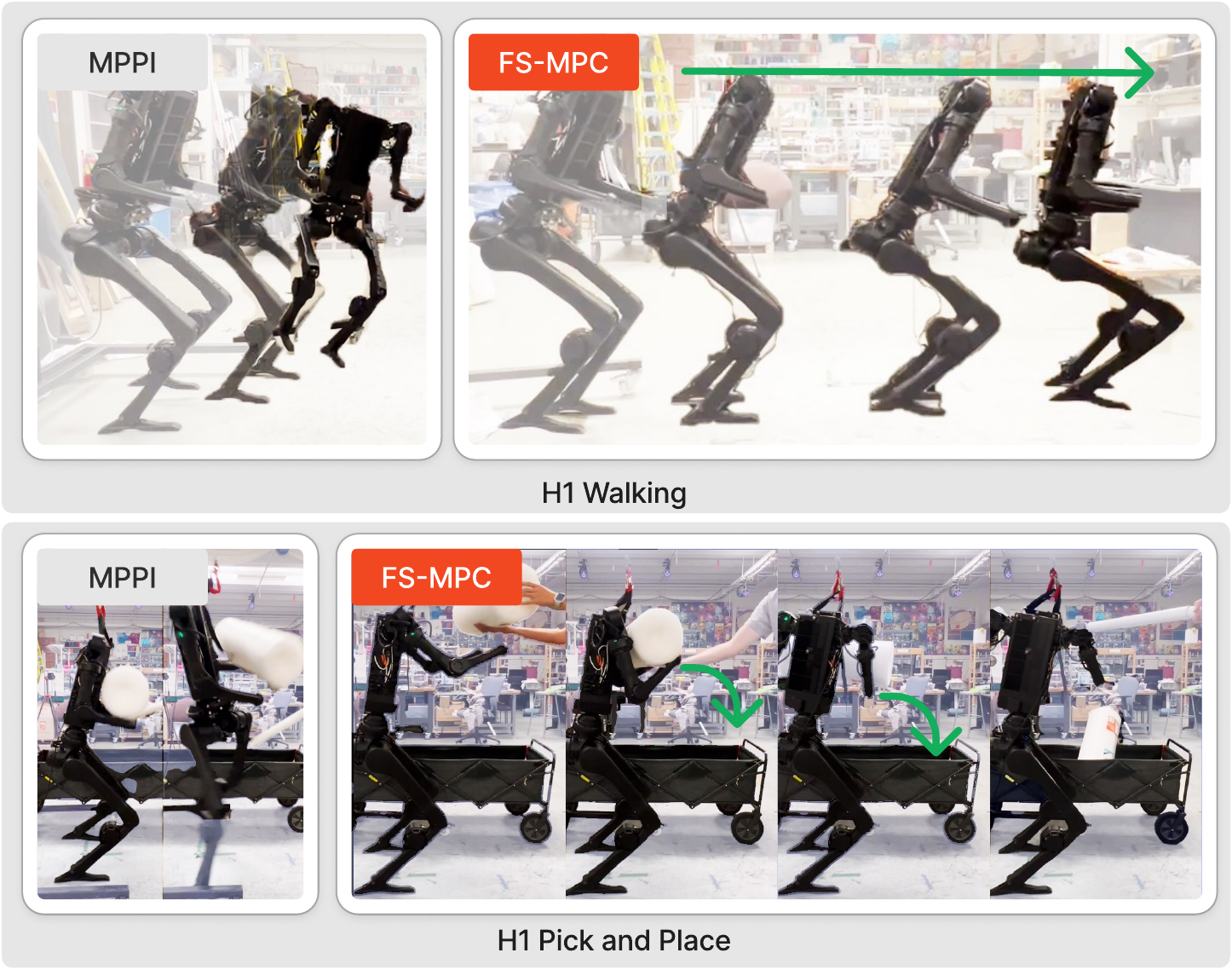}
    \caption{FS-MPC on real robot. Upper: H1 humanoid robot walking from one side of the room to the other. Lower: H1 humanoid robot holding a cylinder and put it into a cart. FS-MPC can finish the task while MPPI failed due to unstable sampling. }
    \label{fig:realworld}
\end{figure}

\subsection{Implementation Details}
\label{subsec:exp_impl}

To realize the algorithm in real time, we made a few algorithmic tweaks to enable efficient and stable online computation.

\noindent \textbf{Hyperparameter choice.}
For hybrid feedback sampling, we set the local search sample count to $N_\text{local} = 8$ and the global search sample count to $N_\text{global} = N - N_\text{local} = 24$ for all tasks.
We set softmax temperature $\lambda = 1.0$ for all tasks.

\noindent \textbf{Real robot setup.} We use a Unitree H1 humanoid robot for real-world experiments. For perception, we combine the robot's proprioception with an external Vicon motion capture system to get the robot's state. The robot state is then fed into the MuJoCo MPC framework~\cite{howellPredictiveSamplingRealtime2022} to generate the reference trajectory. The target joint position command is then sent to the robot through a low-level controller.

\section{Conclusion and Limitations}
\label{sec:limitations}
This paper introduces \methodname{}, a general feedback-based sampling rule that resolves the sampling inefficiency of MPPI.
By leveraging hybrid feedback sampling, \methodname{} achieves superior performance compared to using a feedback policy alone. We demonstrate its effectiveness in controlling highly unstable, contact-rich tasks, both in simulation and on real-world robotic systems.
Future work can further enhance the scalability of \methodname{} by integrating learned components such as value functions or model-based estimators into the sampling process. These additions could help reduce variance, improve computational efficiency, and make \methodname{} more suitable for real-time applications in high-dimensional robotic systems.

\section*{Acknowledgment}
This project is partially funded by NSF Award 2512805. Guannan Qu is
supported by the Pennsylvania Infrastructure Technology Alliance, NSF Grant
2154171, CAREER Award 2339112, the Carnegie Mellon University Manufacturing
Futures Institute, and Jane Street. Zeji Yi acknowledges support from the Wei
Shen and Xuehong Zhang Presidential Fellowship from the College of Engineering
at Carnegie Mellon University and is partially funded by the Amazon AI PhD
Fellowship. Chaoyi Pan acknowledges support from the William J. Happel
Fellowship in the Department of Electrical and Computer Engineering at Carnegie
Mellon University and the 2024--2025 Hsu Chang Memorial Fellowship at Carnegie
Mellon University. Guanya Shi holds concurrent appointments as an Assistant
Professor at Carnegie Mellon University and as an Amazon Scholar. This paper
describes work performed at Carnegie Mellon University and is not associated
with Amazon.

\bibliographystyle{IEEEtran}
\bibliography{refs}

\clearpage
\onecolumn
\section*{Appendix}

\subsection*{Notation}

\begin{table*}[htbp]
  \renewcommand{\arraystretch}{1.6} %
  \centering
  \begin{tabular}{lll}
    \toprule
    \textbf{Symbol} & \textbf{Meaning}                        & \textbf{Definition}                              \\
    \midrule
    $n$             & state dimension                         &                                                  \\
    $m$             & control dimension                       &                                                  \\
    $A_t$           & Discrete-time system matrix at time $t$ & $A_t \in \mathbb{R}^{n \times n}$                \\
    $B_t$           & Discrete-time input matrix at time $t$  & $B_t \in \mathbb{R}^{n \times m}$                \\
    $Q_t$           & state cost matrix at time $t$           & $Q_t \in \mathbb{R}^{n \times n}, Q_t \succeq 0$ \\
    $R_t$           & control cost matrix at time $t$         & $R_t \in \mathbb{R}^{m \times m}, R_t \succeq 0$ \\
    $H$             & finite time horizon                     &                                                  \\
    $P_t$           & cost-to-go matrix at time $t$           & solution from backward-pass LQR                  \\
    $K_t$           & feedback gain matrix                    & solution from backward-pass LQR                  \\
    $Z_t$           & scale matrix for stochastic control     & solution from backward-pass LQR                  \\
    $x_t$           & state at time $t$                       &                                                  \\
    $u_t$           & control input at time $t$               & $u_t = -K_t x_t + Z_t w_t$                       \\
    $\bar{u}_t$     & nominal control input at time $t$       &                                                  \\
    $\Sigma_t$      & covariance matrix for noise at time $t$ &                                                  \\
    $w_t$           & noise vector                            & $w_t \sim \mathcal{N}(0, \sigma^2 I)$            \\
    $X$             & state sequence                          & $X = [x_1, x_2, \ldots, x_{H+1}]^\top$           \\
    $U$             & full control sequence                   & $U = [u_1, u_2, \ldots, u_H]^\top$               \\
    $\lambda$       & inverse temperature parameter in MPPI   &                                                  \\
    $N$             & sample size                             &                                                  \\
    $[V]_i$         & $i$-th sample of control sequence       & $V_i = [v_1, v_2, \ldots, v_H]^\top$             \\
    \bottomrule
  \end{tabular}
  \caption{Notation table}
  \label{tab:notation}
\end{table*}

\subsection*{Proof of MPPI Sample Complexity}

\begin{delayedproof}{thm:mppi_complexity}
  Without loss of generality, we here assume that $U^* = U_t$, which means that the mean of the Gaussian distribution is just the optimal point, guarantees a higher probability to getting more optimal points. By applying the unit vector $e_0 =[\vec{u_0},\mathbf{0},\cdots,\mathbf{0}]^\top$ to the cost function, one can get
  \begin{align*}
    J(U)|_{U=e_0} & = e_0^\top D e_0 = \sum_{t=1}^H (A^{t-1} B \vec{u_0})^\top Q (A^{t-1} B \vec{u_0})  + \vec{u_0}^T R \vec{u_0} \\
                  & \geq \|Q^{-1}\|_2^{-1} \|B^{-1}\|_2^{-2}  \sum_t y^T (A^\top A)^{t-1} y
  \end{align*}
  Here $y = \frac{B \vec{u_0}}{\|B \vec{u_0}\|}$,
  and the inequality comes from that  $Q \succ 0$, and the system is controllable and $B$ is the full rank. So the minimum eigen value of $Q,B$ can be denoted as $\|Q^{-1}\|_2^{-1}> 0$ and $\|B^{-1}\|_2^{-1}> 0$. Further, there exist $\vec{u_0}$ so that $A y = \|A\|_2^2 y$. Therefore, because
  \begin{align*}
    \| \sum_t (A^\top A)^{t-1} \|_2^2 = \sum_t \|A^\top A\|_2^{2t-2}.
  \end{align*}

  We get the lower bounds of $D$'s two norm  $\|D\|_2$ that
  \begin{align*}
    \|D\|_2^2 \geq \frac{\|A^\top A\|_2^{2H}-1}{\|A^\top A\|_2^2-1} \|Q^{-1}\|_2^{-1} \|B^{-1}\|_2^{-2}
  \end{align*}

  Denote the associated unit eigenvector of $D$'s largest eigen with $v_{\max}(D)$. Then:
  $$
    J(U)
    \;=\;
    \tfrac12\,U^T D\,U
    \;\ge\;
    \tfrac12\,\|D\|_2^2\,\bigl(U^T v_{\max}(D)\bigr)^2.
  $$
  \smallskip\noindent\textbf{Step 1: Single-sample bound.}
  To have $J(U) \le H \varepsilon$, it is necessary that
  \begin{equation}
    \tfrac12\,\|D\|_2^2\bigl(U^T v_{\max}\bigr)^2 \;\le\; H\varepsilon
    \quad\Longrightarrow\quad
    \bigl|U^T v_{\max}\bigr|
    \;\le\;
    \sqrt{\frac{2H\,\varepsilon}{\|D\|_2^2}}
    \;\approx\;
    \sqrt{\frac{2H \|Q^{-1}\| \|B^{-1}\|^2\varepsilon}{\|A^\top A\|^{2H-2}}}.
    \label{eq:max_eig}
  \end{equation}
  Given $\mathcal{U} \sim \mathcal{N}(0,I)$, $\mathcal{U}^T v_{\max} \sim \mathcal{N}(0,1)$,
  $$
    \Pr\bigl\{ |\mathcal{U}^T v_{\max}| \le r \bigr\}
    \;\le\;
    2\,r/\sqrt{2\pi}
    \quad
    \forall r \in \mathbb{R}
  $$
  Here we take the value according to \cref{eq:max_eig}, so
  \begin{align*}
    \Pr\bigl\{ J(\mathcal{U}) \le H\varepsilon \bigr\}
     & \;\le\;
    \Pr\bigl\{ |\mathcal{U}^T v_{\max}| \le \sqrt{\frac{2H\,\varepsilon}{\|D\|_2^2}} \bigr\}
    \;\le\;
    O\!\bigl(\frac{\sqrt{H \|Q^{-1}\| \|B^{-1}\|^2\varepsilon}}{\|A^\top A\|^{H-1}}.\bigr),
  \end{align*}

  Moreover, we argue that because $\|D\| = O(\|A ^\top A\|^H)$, for any fixed $\lambda > 0$. There exists $H$, such that $\|D\| \gg \lambda$. Therefore, it is reasonable to take $\lim_{\lambda \rightarrow0}$ as an example to anlyze.

  \smallskip\noindent\textbf{Step 2: Multi-sample bound.}
  Here we use a simplified model to analyze $J(U_{\text{out}})$. That is to analyze $\min J(V_k)$. and we can conclude that
  \begin{align*}
    \Pr\!\left(\min_{i\le N} J\!\left([V_k]_i\right) \le H\varepsilon\right)
    \;=\;
    1 - \bigl(1 - \Pr(J(V_k) \le H\varepsilon)\bigr)^{N}.
  \end{align*}
  In order to let $\Pr(\min J(V_k) \le \varepsilon) \ge 1 - \delta$, we need to have
  \begin{align*}
    \bigl(1 - \Pr(J(V_k) \le H\varepsilon)\bigr)^{N} \;\le\; \delta.
  \end{align*}
  Therefore,
  \begin{align*}
    N
    \;\ge\;
    \frac{\log(1/\delta)}{-\log\bigl(1 - \Pr(J(V_k) \le H\varepsilon)\bigr)}
    \;=\;
    \frac{\log(1/\delta)}{\log\!\Bigl(\frac{1}{1-\Pr(J(V_k)\le H\varepsilon)}\Bigr)}.
  \end{align*}
  Plugging the single-sample upper bound from Step~1,
  \[
    \Pr(J(V_k)\le H\varepsilon)
    \;\le\;
    O\!\bigl(\tfrac{\sqrt{H \|Q^{-1}\| \|B^{-1}\|^2\varepsilon}}{\|A^\top A\|^{H-1}}\bigr),
  \]
  yields an exponential-in-$H$ sample complexity (up to polynomial factors in $H$ and $\varepsilon$):
  \[
    N
    \;\ge\;
    \Theta\!\left(
    \frac{\|A\|_2^{2H-2}}{\sqrt{H \varepsilon}}
    \log\frac{1}{\delta}
    \right)
  \]

\end{delayedproof}

\subsection*{Proof of Sample Complexity with Optimized Sampling Distribution}

\begin{delayedproof}{thm:stability_optimal_proposal}
  We first clarify the quadratic form and the matrix $D$.
  Consider the LTV dynamics over horizon $H$,
  $x_{h+1}=A_h x_h + B_h u_h$, and stack the states and controls as
  \[
    X := [x_1^\top,\ldots,x_{H+1}^\top]^\top\in\mathbb{R}^{n(H+1)},
    \qquad
    U := [u_0^\top,\ldots,u_{H-1}^\top]^\top\in\mathbb{R}^{mH}.
  \]
  Then there exists a (block lower-triangular) matrix $M\in\mathbb{R}^{n(H+1)\times mH}$ and a vector $b$ (depending on $x_0$)
  such that $X = M U + b$.
  More explicitly, if we stack
  \(
    X_0 := [x_0^\top,\ldots,x_H^\top]^\top\in\mathbb{R}^{n(H+1)}
  \)
  and
  \(
    U := [u_0^\top,\ldots,u_{H-1}^\top]^\top\in\mathbb{R}^{mH},
  \)
  then $X_0 = M U + \bar b$ for some $\bar b$ depending on $x_0$, where $M$ is block lower-triangular with blocks
  $M_{ij}\in\mathbb{R}^{n\times m}$ (row $i$ corresponds to $x_{i-1}$ and column $j$ corresponds to $u_{j-1}$) given by
  \[
    M_{ij}
    \;=\;
    \begin{cases}
      A_{i-2}A_{i-3}\cdots A_{j}\,B_{j-1}, & i>j,\\
      0, & i\le j,
    \end{cases}
    \qquad i,j\in\{1,\ldots,H\},
  \]
  with the convention that an empty product equals the identity. In the time-invariant case ($A_h\equiv A$, $B_h\equiv B$),
  this reduces to $M_{ij} = A^{\,i-j-1}B$ for $i>j$ and $M_{ij}=0$ for $i\le j$.
  Let
  \[
    Q := \mathrm{blkdiag}(Q_1,\ldots,Q_{H+1}),
    \qquad
    R := \mathrm{blkdiag}(R_0,\ldots,R_{H-1}),
  \]
  so that the quadratic cost can be written as $J(U)=X^\top Q X + U^\top R U$.
  Expanding in $U$ gives
  \[
    J(U)= U^\top (M^\top Q M + R) U \;+\; 2\,b^\top Q M\,U \;+\; b^\top Q b.
  \]
  Around an optimal open-loop solution $U^\star$, the linear and constant terms can be absorbed by shifting, yielding
  \[
    J(U)= (U-U^\star)^\top D\,(U-U^\star),
    \qquad
    D := M^\top Q M + R \;\succ\; 0,
  \]
  where $D$ is the Hessian of $J$ with respect to the stacked control $U$ (positivity holds, e.g., if $R\succ 0$).
  We proceed in steps.

  \smallskip\noindent\textbf{Step 1: Reduce the objective to $\mathrm{tr}(D\Sigma)$.}
  Let $Z:=V_k-U^\star$. Under $V_k\sim\mathcal{N}(U^\star,\Sigma)$, we have $Z\sim\mathcal{N}(0,\Sigma)$ and
  \[
    J(V_k) \;=\; Z^\top D Z.
  \]
  Using the standard identity for centered Gaussians and quadratic forms,
  \begin{equation}
    \label{eq:EZtr}
    \mathbb{E}[Z^\top D Z] \;=\; \mathrm{tr}(D\Sigma).
  \end{equation}
  Hence the covariance optimization is equivalent to
  \begin{equation}
    \label{eq:opt_tr}
    \min_{\Sigma\succeq 0,\;\det(\Sigma)=1}\;\mathrm{tr}(D\Sigma).
  \end{equation}

  \smallskip\noindent\textbf{Step 2: Solve \eqref{eq:opt_tr} under $\det(\Sigma)=1$.}
  Consider the Lagrangian with the constraint expressed as $\log\det(\Sigma)=0$:
  \[
    \mathcal{L}(\Sigma,\nu) \;=\; \mathrm{tr}(D\Sigma) - \nu \log\det(\Sigma).
  \]
  For $\Sigma\succ 0$, the first-order stationarity condition is
  \[
    \nabla_\Sigma \mathcal{L}(\Sigma,\nu) \;=\; D - \nu \Sigma^{-1} \;=\; 0,
  \]
  which yields $\Sigma=\nu D^{-1}$. Enforcing $\det(\Sigma)=1$ gives
  \[
    1 \;=\; \det(\nu D^{-1}) \;=\; \nu^{mH} \det(D^{-1})
    \;=\; \frac{\nu^{mH}}{\det D},
  \]
  so $\nu = (\det D)^{1/(mH)}$. Define the horizon-normalized constant
  \[
    \rho \;:=\; m\,(\det D)^{1/(mH)}.
  \]
  Therefore
  \[
    \Sigma_k^\star \;=\; \frac{\rho}{m}\,D^{-1},
  \]
  which gives the claimed closed form.

  \smallskip\noindent\textbf{Step 3: Derive the exact distribution of $J(V_k)$.}
  Let $z\sim\mathcal{N}(0,I_{mH})$ and write
  \[
    Z \;=\; \sqrt{\frac{\rho}{m}}\,D^{-1/2} z,
  \]
  which has covariance $\mathbb{E}[ZZ^\top]=\frac{\rho}{m} D^{-1}=\Sigma_k^\star$. Then
  \[
    J(V_k) \;=\; Z^\top D Z \;=\; \frac{\rho}{m}\, z^\top z.
  \]
  Since $z^\top z\sim \chi^2_{mH}$, we obtain
  \[
    J(V_k)\overset{d}{=}\frac{\rho}{m}\,\chi^2_{mH}.
  \]

  \smallskip\noindent\textbf{Step 4: Compute mean and variance.}
  From the representation $J(V_k)\overset{d}{=}\frac{\rho}{m}\,\chi^2_{mH}$ and the standard moments of $\chi^2_{mH}$,
  \[
    \mathbb{E}[\chi^2_{mH}]=mH,\qquad \mathrm{Var}(\chi^2_{mH})=2mH,
  \]
  we immediately get
  \[
    \mathbb{E}[J(V_k)] \;=\; \rho\,H = \Theta(H)
    \qquad
    \mathrm{Var}(J(V_k)) \;=\; 2\,\frac{\rho^2}{m}\,H = \Theta(H)
  \]
  In particular, for fixed control dimension $m$, both the mean and variance scale as $\Theta(H)$.
  Moreover, in common LQR settings, the determinant-geometric-mean factor $(\det D)^{1/(mH)}$ (and hence $\rho$) can be treated as
  $\Theta(1)$ with respect to $H$;
  see Lemma~\ref{lem:det_factorization_riccati} and Remark~\ref{rem:steady_state_det_rate} for a representative LTI/normalized-cost case
  (where $Q=I$ and $R=I$, so $D=M^\top Q M + R = I + M^\top M = \bar D$). In particular, if the associated Riccati recursion
  \(\,P_{k+1}=I + A P_k A^\top - A P_k B(I+B^\top P_k B)^{-1}B^\top P_k A^\top\,\) converges to a fixed point $P^\star\succeq 0$
  satisfying the algebraic Riccati equation
  \[
    P^\star
    \;=\;
    I + A P^\star A^\top
    \;-\;
    A P^\star B\bigl(I + B^\top P^\star B\bigr)^{-1}B^\top P^\star A^\top,
  \]
  then \((\det \bar D)^{1/(mH)}\to \det\!\bigl(I + B^\top P^\star B\bigr)^{1/m}\), and thus
  \[
    \rho \;=\; m(\det \bar D)^{1/(mH)} \;\longrightarrow\; m\,\det\!\bigl(I + B^\top P^\star B\bigr)^{1/m},
  \]
  a constant depending only on $(A,B)$ (via $P^\star$) and independent of $H$.

  \smallskip\noindent\textbf{Step 5: Probability $\Pr\{J(V_k)\le H \varepsilon\}$ is a constant (with a quantitative asymptotic).}
  We here choose $\varepsilon = \rho$ to make the probability a constant.
  By $J(V_k)\overset{d}{=}\frac{\rho}{m}\,\chi^2_{mH}$,
  \[
    \Pr\{J(V_k)\le H \varepsilon\}
    \;=\;
    \Pr\!\left\{\tfrac{\rho}{m}\,\chi^2_{mH} \le H \varepsilon\right\}
    \;=\;
    \Pr\!\left\{\chi^2_{mH} \le mH\,\varepsilon/\rho\right\}.
  \]
  With $\varepsilon=\rho$, this becomes $\Pr\{\chi^2_{mH}\le mH\}$.
  Let $z_1,\ldots,z_{mH}\stackrel{\text{i.i.d.}}{\sim}\mathcal{N}(0,1)$ so that
  $\chi^2_{mH}=\sum_{i=1}^{mH} z_i^2$. Define centered i.i.d.\ variables
  \[
    Y_i \;:=\; z_i^2-1,\qquad \mathbb{E}[Y_i]=0,\qquad \mathrm{Var}(Y_i)=2.
  \]
  Then
  \[
    \frac{\chi^2_{mH}-mH}{\sqrt{2mH}} \;=\; \frac{\sum_{i=1}^{mH} Y_i}{\sqrt{2mH}}.
  \]
  By the Berry--Esseen theorem (applied to $\{Y_i\}$, which have finite third absolute moment),
  there exists a universal constant $C_{\mathrm{BE}}>0$ such that for all $H\ge 1$,
  \begin{equation}
    \label{eq:BE}
    \sup_{x\in\mathbb{R}}
    \left|
    \Pr\!\left\{\frac{\chi^2_{mH}-mH}{\sqrt{2mH}}\le x\right\}-\Phi(x)
    \right|
    \;\le\;
    \frac{C_{\mathrm{BE}}}{\sqrt{mH}},
  \end{equation}
  where $\Phi$ is the standard normal CDF. Evaluating at $x=0$ gives
  \[
    \left|\Pr\{\chi^2_{mH}\le mH\}-\Phi(0)\right|
    \;\le\; \frac{C_{\mathrm{BE}}}{\sqrt{mH}}.
  \]
  Since $\Phi(0)=1/2$, we obtain
  \[
    \Pr\{\chi^2_{mH}\le mH\} \;=\; \frac12 + O\!\left(\frac{1}{\sqrt{mH}}\right),
  \]
  which gives the stated quantitative asymptotic. In particular, choose any $d_0$ large enough that
  $C_{\mathrm{BE}}/\sqrt{d_0}\le 1/6$ and set $p_0:=1/3$; then for all $H$ such that $mH\ge d_0$,
  \[
    \Pr\{\chi^2_{mH}\le mH\}\;\ge\;\frac12-\frac{C_{\mathrm{BE}}}{\sqrt{mH}}
    \;\ge\;\frac12-\frac16\;=\;\frac13 = \Theta(1)
  \]
  yielding a constant lower bound. Therefore,
  \[
    \Pr\{J(V_k)\le H \varepsilon\} \;\ge\; \frac13 = \Theta(1).
  \]

\smallskip\noindent\textbf{Step 6: Sample complexity for one success with high probability.}
  Let $E_i:=\{J(V_k^{(i)})\le H \varepsilon\}$. Under i.i.d.\ sampling,
  $\Pr(E_i)=p_H:=\Pr\{J(V_k)\le H\varepsilon\}$ and the events $\{E_i\}$ are independent.
  Thus
  \[
  \Pr\left\{\exists i\le N:\;E_i\right\}
  \;=\; 1-\Pr\left\{\cap_{i=1}^{N} E_i^c\right\}
  \;=\; 1-(1-p_H)^{N}.
  \]
  From Step~5, for all $H$ such that $mH\ge d_0$ we have $p_H\ge p_0$ (with $p_0:=1/3$). Hence
  \[
  \Pr\left\{\exists i\le N:\;J(V_k^{(i)})\le H\varepsilon\right\}
  \;\ge\; 1-(1-p_0)^{N}.
  \]
  To make the right-hand side at least $1-\delta$, it suffices that
  \[
  (1-p_0)^{N} \;\le\; \delta
    \quad\Longleftrightarrow\quad
  N \;\ge\; \frac{\log(1/\delta)}{-\log(1-p_0)} \ge \frac{\log(1/\delta)}{\log(\tfrac{3}{2})} = \Theta(\log(1/\delta)).
  \]
  which gives the stated logarithmic-in-$1/\delta$ sample complexity. This concludes the proof.
\end{delayedproof}
\begin{lemma}[Geometric-mean determinant factorization via a Riccati recursion (LTI, normalized cost)]
  \label{lem:det_factorization_riccati}
  Let $A\in\mathbb{R}^{n\times n}$, $B\in\mathbb{R}^{n\times m}$ with $\mathrm{rank}(B)=m$, and horizon $H\ge 1$.
  Define the strictly block-lower matrix $M\in\mathbb{R}^{(nH)\times(mH)}$ by block entries
  \[
  M_{ij}=\begin{cases}
  A^{\,i-j-1}B, & i>j,\\
  0,& i\le j,
  \end{cases}
  \qquad i,j\in\{1,\dots,H\},
  \]
  and define
  \[
  \bar D \;:=\; I_{mH} + M^\top M \;\in\;\mathbb{S}_{++}^{mH}.
  \]
  Let $\{P_k\}_{k=0}^{H}\subset\mathbb{S}_+^n$ be generated by
  \begin{align}
  P_0 &:= 0, \label{eq:P0}\\
  S_k &:= I_m + B^\top P_k B, \label{eq:Sk_def}\\
  P_{k+1} &:= I_n + A P_k A^\top \;-\; A P_k B\, S_k^{-1}\, B^\top P_k A^\top,
  \qquad k=0,1,\dots,H-1. \label{eq:riccati_update}
  \end{align}
  Then the determinant admits the factorization
  \[
  \det(\bar D)
  =
  \prod_{k=0}^{H-1}\det(S_k),
  \]
  and hence
  \[
  \boxed{
  \det(\bar D)^{1/H}
  =
  \left(\prod_{k=0}^{H-1}\det\!\big(I_m+B^\top P_k B\big)\right)^{1/H}.
  }
  \]
  \end{lemma}
  
  \begin{proof}
  \smallskip\noindent\textbf{Step 1: Lifted dynamics representation.}
  Consider the linear system
  \[
  x_{t+1}=A x_t + B u_t,\qquad x_0=0,
  \]
  with inputs $u_t\in\mathbb{R}^m$ and states $x_t\in\mathbb{R}^n$ for $t=0,\dots,H-1$.
  Define stacked vectors
  \[
  u:=\begin{bmatrix}u_0\\ \vdots\\ u_{H-1}\end{bmatrix}\in\mathbb{R}^{mH},
  \qquad
  x:=\begin{bmatrix}x_1\\ \vdots\\ x_H\end{bmatrix}\in\mathbb{R}^{nH}.
  \]
  By unrolling the recursion,
  \[
  x_{t+1}=\sum_{j=0}^{t}A^{t-j}B u_j,
  \]
  so the lifted map $u\mapsto x$ is linear and can be written as
  \[
  x = M u,
  \]
  where $M$ is exactly the strictly block-lower matrix defined in the lemma.
  
  \smallskip\noindent\textbf{Step 2: Schur complement reduction.}
  Introduce the symmetric block matrix
  \[
  \mathcal{K}
  :=
  \begin{bmatrix}
  I_{mH} & M^\top\\
  M & I_{nH}
  \end{bmatrix}.
  \]
  Since $I_{nH}\succ 0$, the Schur complement formula yields
  \begin{equation}
  \label{eq:schur_det}
  \det(\mathcal{K})
  =
  \det(I_{nH})\,\det\!\big(I_{mH} + M^\top I_{nH}^{-1}M\big)
  =
  \det(\bar D).
  \end{equation}
  Thus it suffices to factorize $\det(\mathcal{K})$.
  
  \smallskip\noindent\textbf{Step 3: Sequential elimination gives Riccati and determinant factors.}
  Perform block Gaussian elimination on $\mathcal{K}$ in time order (eliminate $(u_0,x_1),(u_1,x_2),\ldots$).
  The resulting Schur complements on the state blocks yield the Riccati recursion \eqref{eq:riccati_update},
  and each elimination step contributes a multiplicative determinant factor $\det(S_k)$, where $S_k=I_m+B^\top P_k B$.
  Consequently,
  \[
    \det(\mathcal{K})=\prod_{k=0}^{H-1}\det(S_k).
  \]
  Combining with \eqref{eq:schur_det} yields $\det(\bar D)=\prod_{k=0}^{H-1}\det(S_k)$.

  \smallskip\noindent\textbf{Step 4: Take geometric means.}
  Taking $H$-th roots gives the stated identity for $\det(\bar D)^{1/H}$.
  \end{proof}
  
  \begin{remark}[Steady-state limit]
  \label{rem:steady_state_det_rate}
  Assume the Riccati recursion \eqref{eq:riccati_update} admits a finite fixed point $P^\star\succeq 0$
  and that $P_k\to P^\star$ as $k\to\infty$.
  Define $S^\star := I_m + B^\top P^\star B$.
  Then, as $H\to\infty$,
  \[
  \det(\bar D)^{1/H}
  =
  \left(\prod_{k=0}^{H-1}\det(S_k)\right)^{1/H}
  \longrightarrow
  \det(S^\star)
  =
  \det\!\big(I_m + B^\top P^\star B\big).
  \]
  Equivalently, the control-dimension-normalized geometric mean
  \[
    \det(\bar D)^{1/(mH)} \;=\; \bigl(\det(\bar D)^{1/H}\bigr)^{1/m}
    \;\longrightarrow\; \det(S^\star)^{1/m},
  \]
  which is a constant depending only on $(A,B)$ through the steady-state Riccati solution $P^\star$, and is independent of the horizon.
  \end{remark}
  
\subsection*{Proof of Feedback Sampling Equivalence to Optimal Sampling Distribution Optimization}

\begin{delayedproof}{thm:equivalence}
  First, observe that in Feedback MPPI, each control \(u_t\) is defined by
  \[
    u_t = -K_t x_t + Z_t w_t,
    \quad x_{t+1} = A x_t + B u_t,
    \quad w_t \sim \mathcal{N}(0, I).
  \]
  One can rewrite the full control sequence as
  \[
    U = C\,Z\,W,
  \]
  where \(Z = \mathrm{diag}(Z_1, \ldots, Z_T)\), \(W = \begin{bmatrix}w_1^\top & \cdots & w_T^\top\end{bmatrix}^\top\), and \(C\) collects all the linear feedback relationships. Therefore, the control distribution is Gaussian:
  \[
    U \sim \mathcal{N}(0,\,C\,Z\,Z^\top\,C^\top).
  \]

  Next, CoVO solves
  \[
    \min_{L} \; \mathbb{E}\!\Bigl[\tfrac{1}{2} U^\top D\,U\Bigr]
    \quad
    \text{subject to}
    \quad
    U = L W,\quad L^\top L = I,
  \]
  where \(D = M^\top Q\,M + R\), and \(M\), \(Q\), \(R\) define the stage and terminal costs in block-diagonal form. Because
  \[
    \mathbb{E}\Bigl[\tfrac{1}{2} U^\top D\,U \Bigr]
    = \tfrac{1}{2} \,\mathrm{Tr}\bigl(D\,\mathbb{E}[UU^\top]\bigr),
  \]
  the optimization is equivalent to finding \(L\) that minimizes \(\tfrac{1}{2}\,\mathrm{Tr}(D\,LL^\top)\) under \(L^\top L = I\). The minimizer \(L^*\) satisfies
  \[
    L^{*\top} D \,L^* = I,
    \quad
    L^* = D^{-1/2} \,\bigl(\sqrt{D}\bigr).
  \]
  Hence, the optimal control distribution is
  \[
    U \sim \mathcal{N}\bigl(0,\;L^* L^{*\top}\bigr).
  \]

  To see that Feedback MPPI matches the CoVO solution, verify that \(Z^T C^\top\,D\,C\,Z = I\). Under LQR-based gains \(\{K_t\}\) and scale matrices \(\{Z_t\}\), one obtains
  \[
    (C Z)^\top \,D\, (C Z) \;=\; I,
  \]
  which implies \(C\,Z\) is a valid CoVO optimizer (it satisfies the same constraint as \(L^*\)). Consequently, the distribution \(U \sim \mathcal{N}(0,C\,Z\,Z^\top C^\top)\) is the same as \(U \sim \mathcal{N}(0,L^*L^{*\top})\). Therefore, when \(\lambda \to 0\) and \(N \to \infty\) and the mean is the LQR-optimal mean control, Feedback MPPI recovers the CoVO distribution exactly.

  Another way to see the equivalence is via the classical insight that any stationary Gaussian process can be represented as the output of a linear system driven by white noise. In Feedback MPPI, the sequence of controls \(U\) emerges from propagating white noise \(W\) through a linear transformation \(C\,Z\). Simultaneously, the CoVO framework seeks a linear map \(L^*\) such that \(\mathbb{E}\bigl[U^\top D\,U \bigr]\) is minimized under \(L^{*\top} L^* = I\). Since one can interpret the LQR-based recursion as producing the unique factor \(C\,Z\) that satisfies \((C\,Z)^\top D\,(C\,Z) = I\), it follows that \(C\,Z\) coincides with the CoVO solution \(L^*\). Hence, Feedback MPPI (with noise scaling given by \(\{Z_t\}\) and mean given by the LQR solution) yields exactly the CoVO-optimal Gaussian distribution for \(U\).

\end{delayedproof}

\end{document}